\documentclass{article}

\pdfoutput=1

\usepackage{microtype}
\usepackage{graphicx}
\usepackage{subfigure}
\usepackage{booktabs} 
\usepackage{bm}
\usepackage{amsmath, amssymb, amsthm}
\usepackage{algorithm,algorithmic}
\usepackage{mathtools}
\usepackage{adjustbox}
\usepackage{enumitem}
\usepackage{mdwlist}
\usepackage{physics}
\usepackage{framed}
\usepackage{dsfont}
\usepackage{tikz}
\usepackage{tkz-graph}
\usepackage{pgfplots}
\usepackage{multirow}
\usepackage{mdwtab}
\usetikzlibrary{positioning,shapes,backgrounds,arrows.meta,positioning}
\usepackage{hyperref}

\usepackage[capitalize]{cleveref}
\newtheorem{thm}{Theorem}
\newtheorem{lemma}{Lemma}

\newcommand{\minimize}{\operatorname*{minimize}}
\newcommand{\maximize}{\operatorname*{maximize}}
\newcommand{\bfx}{\mathbf{x}}

\newcommand{\bfz}{\mathbf{z}}
\newcommand{\bfh}{\mathbf{h}}
\newcommand{\bfc}{\mathbf{c}}
\newcommand{\diag}{\mathop{\rm diag}}
\newcommand{\bftheta}{\boldsymbol\theta}
\newcommand{\argmin}{\mathop{\rm argmin}}

\renewcommand{\algorithmicrequire}{\textbf{initialize}}

\usepackage[accepted]{icml2018}

\icmltitlerunning{Efficient end-to-end learning for quantizable representations}

\begin{document}

\twocolumn[
\icmltitle{Efficient end-to-end learning for quantizable representations}

\icmlsetsymbol{equal}{*}

\begin{icmlauthorlist}
\icmlauthor{Yeonwoo Jeong}{snu}
\icmlauthor{Hyun Oh Song}{snu}
\end{icmlauthorlist}

\icmlaffiliation{snu}{Department of Computer Science and Engineering, Seoul National University, Seoul, Korea}

\icmlcorrespondingauthor{Hyun Oh Song}{hyunoh@snu.ac.kr}

\vskip 0.3in
]

\printAffiliationsAndNotice{}  

\begin{abstract}
    Embedding representation learning via neural networks is at the core foundation of modern similarity based search. While much effort has been put in developing algorithms for learning binary hamming code representations for search efficiency, this still requires a linear scan of the entire dataset per each query and trades off the search accuracy through binarization. To this end, we consider the problem of directly learning a quantizable embedding representation and the sparse binary hash code end-to-end which can be used to construct an efficient hash table not only providing significant search reduction in the number of data but also achieving the state of the art search accuracy outperforming previous state of the art deep metric learning methods. We also show that finding the optimal sparse binary hash code in a mini-batch can be computed exactly in polynomial time by solving a \emph{minimum cost flow} problem. Our results on Cifar-100 and on ImageNet datasets show the state of the art search accuracy in precision@k and NMI metrics while providing up to $98\times$ and $478\times$ search speedup respectively over exhaustive linear search. The source code is available at \href{https://github.com/maestrojeong/Deep-Hash-Table-ICML18}{https://github.com/maestrojeong/Deep-Hash-Table-ICML18}.
\end{abstract} 
\section{Introduction} 
Learning the representations that respect the pairwise relationships is one of the most important problems in machine learning and pattern recognition with vast applications. To this end, deep metric learning methods \cite{contrastive, triplet, facenet, liftedstruct, npairs, facility, seanbell, domaintransduction} aim to learn an embedding representation space such that similar data are close to each other and vice versa for dissimilar data. Some of these methods have shown significant advances in various applications in retrieval \cite{npairs, facenet}, clustering \cite{facility}, domain adaptation \cite{domaintransduction}, video understanding \cite{triplet_video}, etc. 

Despite recent advances in deep metric learning methods, deploying the learned embedding representation in large scale applications poses great challenges in terms of the inference efficiency and scalability. To address this, practitioners in large scale retrieval and recommendation systems often resort to a separate post-processing step where the learned embedding representation is run through quantization pipelines such as sketches, hashing, and vector quantization in order to significantly reduce the number of data to compare during the inference while trading off the accuracy.

In this regard, we propose a novel end-to-end learning method for quantizable representations jointly optimizing for the quality of the network embedding representation and the performance of the corresponding binary hash code. In contrast to some of the recent methods \cite{cao2016deep, liu2017learning}, our proposed method avoids ever having to cluster the entire dataset, offers the modularity to accommodate any existing deep embedding learning techniques \cite{facenet, npairs}, and is efficiently trained in a mini-batch stochastic gradient descent setting. We show that the discrete optimization problem of finding the optimal binary hash codes given the embedding representations can be computed efficiently and exactly by solving an equivalent \emph{minimum cost flow problem}. The proposed method alternates between finding the optimal hash codes of the given embedding representations in the mini-batch and adjusting the embedding representations indexed at the activated hash code dimensions via deep metric learning methods.

Our end-to-end learning method outperforms state-of-the-art deep metric learning approaches \cite{facenet, npairs} in retrieval and clustering tasks on the Cifar-100 \cite{cifar100} and the ImageNet \cite{imagenet} datasets while providing up to several orders of magnitude speedup during inference. Our method utilizes efficient off-the-shelf implementations from OR-Tools (Google optimization tools for combinatorial optimization problems) \cite{ortools}, the deep metric learning library implementation in Tensorflow \cite{tensorflow}, and is efficient to train. The state of the art deep metric learning approaches \cite{facenet, npairs} use the class labels during training (for the hard negative mining procedure) and since our method utilizes the approaches as a component, we focus on the settings where the class labels are available during training. 

\section{Related works} \vspace{-0.3em}
Learning the embedding representation via neural networks dates back to over two decades ago. Starting with Siamese networks \cite{signatureVerification, contrastive}, the task is to learn embedding representation of the data using neural networks so that similar examples are close to each other and dissimilar examples are farther apart in the embedding space. Despite the recent successes and near human performance reported in retrieval and verification tasks \cite{facenet}, most of the related literature on learning efficient representation via deep learning focus on learning the binary hamming codes for finding nearest neighbors with linear search over entire dataset per each query. Directly optimizing for the embedding representations in deep networks for quantization codes and constructing the hash tables for significant search reduction in the number of data is much less studied.

\cite{xia2014supervised} first precompute the hash code based on the labels and trains the embedding to be similar to the hash code. \cite{zhao2015deep} apply element-wise sigmoid on the embedding and minimizes the triplet loss. \cite{hammingmetric} optimize the upper bound on the triplet loss defined on hamming code vectors. \cite{liong2015deep} minimize the difference between the original and the signed version of the embedding with orthogonality regularizers in a network. \cite{li2017deep} employ discrete cyclic coordinate descent \cite{shen2015supervised} on a discrete sub-problem optimizing one hash bit at a time but the algorithm has neither the convergence guarantees nor the bound on the number of iterations. All of these methods focus on learning the binary codes for hamming distance ranking and perform an exhaustive linear search over the entire dataset which is not likely to be suitable for large scale problems.

\cite{cao2016deep} minimize the difference between the similarity label and the cosine distance of network embedding. \cite{liu2017learning} define a distance between a quantized data and continuous embedding, and back-propagates the metric loss error only with respect to the continuous embedding. Both of these methods require \emph{repeatedly} running k-means clustering on the \emph{entire} dataset while training the network at the same time. This is unlikely to be practical for large scale problems because of the prohibitive computational complexity and having to store the cluster centroids for all classes in the memory as the number of classes becomes extremely large \cite{extreme_varma, extreme_langford}.

In this paper, we propose an efficient end-to-end learning method for quantizable representations jointly optimizing the quality of the embedding representation and the performance of the corresponding hash codes in a scalable mini-batch stochastic gradient descent setting in a deep network and demonstrate state of the art search accuracy and quantitative search efficiency on multiple datasets.  \vspace{-0.3em}
\section{Problem formulation} \vspace{-0.3em}
\label{sec:formulation}
Consider a hash function $r(\bfx)$ that maps an input data $\bfx \in \mathcal{X}$ onto a $d$ dimensional binary compound hash code $\bfh \in \{0,1\}^d$ with the constraint that $k$ out of $d$ total bits needs to be set. We parameterize the mapping as \vspace{-1.5em}

\small
\begin{align}
&r(\bfx)=\argmin_{\bfh \in \{0,1\}^d} -f(\bfx; \bftheta)^\intercal \bfh \nonumber\\
&\text{\ \ subject to } ||\bfh||_1 = k,
\label{eqn:hash_function} \vspace{-1em}
\end{align}
\normalsize
where $f(\cdot, \bftheta) :  \mathcal{X} \rightarrow \mathbb{R}^d$ is a transformation (i.e. neural network) differentiable with respect to the parameter $\bftheta$ and takes the input $\bfx$ and emits the $d$ dimensional embedding vector. Given the hash function $r(\cdot)$, we define a hash table $H$ which is composed of $d$ buckets with each bucket indexed by a compound hash code $\bfh$. Then, given a query $\bfx_q$, union of all the items in the buckets indexed by $k$ active bits in $r(\bfx_q)$ is retrieved as the candidates of the approximate nearest items of $\bfx_q$. Finally, this is followed by a reranking operation where the retrieved items are ranked according to the distances computed using the original embedding representation $f(\cdot; \bftheta)$.

Note, in quantization based hashing \cite{survey_learningtohash, cao2016deep}, a set of prototypes or cluster centroids are first computed via dictionary learning or other clustering (i.e. k-means) algorithms. Then, the function $f(\bfx; \bftheta)$ is represented by the indices of k-nearest prototypes or centroids. Concretely, if we replace $f(\bfx; \bftheta)$ in \Cref{eqn:hash_function} with the negative distances of the input item $\bfx$ with respect to all $d$ prototypes or centroids, $[-||\bfx - \bfc_1||_2, \ldots, -||\bfx - \bfc_d||_2]^\intercal$, then the corresponding hash function $r(\bfx)$ can be used to build the hash table. 

In contrast to most of the recent methods that learn a hamming ranking in a neural network and perform exhaustive linear search over the entire dataset \cite{xia2014supervised, zhao2015deep, hammingmetric, li2017deep}, quantization based methods, have guaranteed search inference speed up by only considering a subset of $k$ out $d$ buckets and thus avoid exhaustive linear search. We explicitly maintain the sparsity constraint on the hash code in \Cref{eqn:hash_function} throughout our optimization without continuous relaxations to inherit the efficiency aspect of quantization based hashing and this is one of the key attributes of the algorithm.

Although quantization based hashing is known to show high search accuracy and search efficiency \cite{survey_learningtohash}, running the quantization procedure on the entire dataset to compute the cluster centroids is computationally very costly and requires storing all of the cluster centroids in the memory. Our desiderata are to formulate an efficient end-to-end learning method for quantizable representations which (1) guarantees the search efficiency by avoiding linear search over the entire data, (2) can be efficiently trained in a mini-batch stochastic gradient descent setting and  avoid having to quantize the entire dataset or having to store the cluster centroids for all classes in the memory, and (3) offers the modularity to accommodate existing embedding representation learning methods which are known to show the state of the art performance on retrieval and clustering tasks. \vspace{-0.5em}
\section{Methods} \vspace{-0.3em}
\label{sec:methods}
We formalize our proposed method in \Cref{sec:method_main} and discuss the subproblems in \Cref{sec:method_hash} and in \Cref{sec:method_embedding}.\vspace{-0.3em}
\subsection{\mbox{End-to-end learning for quantizable representations}}
\label{sec:method_main} \vspace{-0.5em}
Finding the optimal set of embedding representations and the corresponding hash codes is a chicken and egg problem. Embedding representations are required to infer which $k$ activation dimensions to set in the corresponding binary hash code, but the binary hash codes are needed to adjust the embedding representations indexed at the activated bits so that similar items get hashed to the same buckets and vice versa. We formalize this notion in \Cref{eqn:master_eqn} below. \vspace{-1.5em}

\small
\begin{align}
&\minimize_{\substack{\boldsymbol\theta\\ \mathbf{h}_1,\ldots,\mathbf{h}_n}}  ~\underbrace{\text{\large $\ell$}_{\text{metric}} (\{f(\mathbf{x}_i; \boldsymbol{\theta})\}_{i=1}^n; \mathbf{h}_1,\ldots,\mathbf{h}_n)}_{\text{embedding representation quality }} ~+ \nonumber\\
&~~~~~~~~~~~~\gamma \underbrace{\left( \sum_i^n -f(\mathbf{x}_i; \boldsymbol\theta)^\intercal \mathbf{h}_i + \sum_{i}^n\sum_{j : y_j \neq y_i} \mathbf{h}_i^\intercal P \mathbf{h}_j \right)}_{\text{hash code performance }} \nonumber\\
&\text{\ \ subject to }~ \mathbf{h}_i \in \{0,1\}^d, || \mathbf{h}_i ||_1 = k, ~ \forall i,
\label{eqn:master_eqn} 
\end{align}
\normalsize
where the matrix $P$ encodes the pairwise cost for the hash code similarity between each negative pair and $\gamma$ is the trade-off hyperparameter balancing the loss contribution between the embedding representation quality given the hash codes and the performance of the hash code with respect to the embedding representations. We solve this optimization problem via alternating minimization through iterating over solving for $k$-sparse binary hash codes $\mathbf{h}_1,\ldots,\mathbf{h}_n$ and updating the parameters of the deep network $\boldsymbol\theta$ for the continuous embedding representations per each mini-batch. Following subsections discuss these two steps in detail. \vspace{-0.5em}
\subsection{Learning the compound hash code} \vspace{-0.5em}
\label{sec:method_hash}
Given a set of continuous embedding representations $\{f(\mathbf{x}_i; \boldsymbol\theta)\}_{i=1}^n$, we seek to solve the following subproblem in \Cref{eqn:k_energy} where the task is to (\emph{unary}) select $k$ as large elements of the each embedding vector as possible, while (\emph{pairwise}) selecting as orthogonal elements as possible across different classes. The unary term mimics the hash function $r(\bfx)$ in \Cref{eqn:hash_function} and the pairwise term has the added benefit that it also provides robustness to the optimization especially during the early stages of the training when the embedding representation is not very accurate. \vspace{-3em}

\small
\begin{align}
&\minimize_{\mathbf{h}_1,\ldots,\mathbf{h}_n}~ \underbrace{\sum_{i}^n -f(\mathbf{x}_i; \boldsymbol\theta)^\intercal \mathbf{h}_i + \sum_i^n \sum_{j : y_j \neq y_i} \mathbf{h}_i^\intercal P \mathbf{h}_j}_{\coloneqq~ g(\bfh_{1, \ldots, n}; \bftheta)} \nonumber\\
&\text{\ \  subject to }~~ \mathbf{h}_i \in \{0,1\}^d, || \mathbf{h}_i ||_1 = k, ~\forall i,
\label{eqn:k_energy} 
\end{align}
\normalsize
However, solving for the optimal solution of the problem in \Cref{eqn:k_energy} is NP-hard in general even for the simple case where $k=1$ and $d>2$ \cite{boykov_fast}. Thus, we construct a upper bound function $\bar{g}(\bfh_{1, \ldots, n}; \bftheta)$ to the objective function $g(\bfh_{1, \ldots, n}; \bftheta)$ which we argue that it can be exactly optimized by establishing the connection to a network flow algorithm. The upper bound function is a slightly reparameterized discrete objective where we optimize the hash codes over the average embedding vectors per each class instead. We first rewrite\footnote{We also omit the dependence of the index $i$ for each $\bfh_k$ and $\bfh_l$ to avoid the notation clutter.} the objective function by indexing over each class and then over each data per class and derive the upper bound function as shown below. \vspace{-1.5em}

\small
\begin{align}
&g(\bfh_{1,\ldots,n}; \bftheta) = \sum_{i}^{n_c}\sum_{k:y_k=i}-f(\mathbf{x}_k; \boldsymbol\theta)^\intercal \mathbf{h}_k + \sum_{i}^{n_c} \sum_{\substack{k:y_k=i,\\l:y_l \neq i}} \mathbf{h}_k^\intercal P \mathbf{h}_l \nonumber\\
&\leq~ \sum_{i}^{n_c}\sum_{k:y_k=i}-\mathbf{c}_i^\intercal \mathbf{h}_k + \sum_{i}^{n_c} \sum_{\substack{k:y_k=i,\\l:y_l \neq i}} \mathbf{h}_k^\intercal P \mathbf{h}_l \nonumber\\
&~~~~~~~~+ \underbrace{\maximize_{\substack{\bfh_1, \ldots, \bfh_n\\\mathbf{h}_i \in \{0,1\}^d, || \mathbf{h}_i ||_1 = k}} \sum_{i=1}^{n_c} \sum_{k:y_k=i}\left(\mathbf{c}_i-f(\mathbf{x}_k; \boldsymbol \theta\right))^\intercal\mathbf{h}_k}_{:= M(\bftheta)} \nonumber\\
& =~ \bar{g}(\bfh_{1, \ldots, n}; \bftheta)
\label{eqn:upperbound} 
\end{align}
\normalsize
where $n_c$ denotes the number of classes in the mini-batch, $m= |\{k:y_k=i\}|$, and $ \mathbf{c}_i =\frac{1}{m}\sum_{k:y_k=i} f(\bfx_k; \bftheta)$. Here, w.l.o.g we assume each class has $m$ number of data in the mini-batch (i.e. Npairs \cite{npairs} mini-batch construction). The last term in upper bound, denoted as $M(\bftheta)$, is constant with respect to the hash codes and is non-negative. Note, from the bound in \Cref{eqn:upperbound}, the gap between the minimum value of $g$ and the minimum value of $\bar{g}$ is bounded above by $M(\bftheta)$. Furthermore,  since this value corresponds to the maximum deviation of an embedding vector from its class mean of the embedding, the bound gap decreases over iterations as we update the network parameter $\bftheta$ to attract similar pairs of data and vice versa for dissimilar pairs in the other embedding subproblem (more details in \Cref{sec:method_embedding}).  

Moreover, minimizing the upper bound over each hash codes $\{\bfh_i\}_{i=1}^n$ is equivalent to minimizing a reparameterization $\hat{g}(\bfz_{1, \ldots, n_c}; \bftheta)$ over $\{\bfz_i\}_{i=1}^{n_c}$ defined below because for a given class label $i$, each $\bfh_k$ shares the same $\bfc_i$ vector. \vspace{-1.5em}

\small
\begin{align}
&\minimize_{\substack{\bfh_1, \ldots, \bfh_n\\\mathbf{h}_i \in \{0,1\}^d, || \mathbf{h}_i ||_1 = k}} ~\sum_{i}^{n_c}\sum_{k:y_k=i}-\mathbf{c}_i^\intercal \mathbf{h}_k+\sum_{i}^{n_c} \sum_{\substack{k:y_k=i,\\l:y_l \neq i}} \mathbf{h}_k^\intercal P \mathbf{h}_l \nonumber\\
&=\minimize_{\substack{\bfz_1, \ldots, \bfz_{n_c}\\\mathbf{z}_i \in \{0,1\}^d, || \mathbf{z}_i ||_1 = k}} ~m \underbrace{\left(\sum_{i}^{n_c}-\mathbf{c}_i^\intercal \mathbf{z}_i+\sum_{i}^{n_c}\sum_{j\neq i} \mathbf{z}_i^\intercal P' \mathbf{z}_j \right)}_{\coloneqq~ \hat{g}(\bfz_{1, \ldots, n_c}; \bftheta)}, \nonumber 
\end{align}
\normalsize
where $P' = m P$. Therefore, we formulate the following optimization problem below whose objective upper bounds the original objective in \Cref{eqn:k_energy} over all feasible hash codes $\{\bfh_i\}_{i=1}^n$. \vspace{-1.5em}

\small
\begin{align}
&\minimize_{\mathbf{z}_1,\ldots,\mathbf{z}_{n_c}}~ \sum_{i}^{n_c} -\mathbf{c}_i^\intercal \mathbf{z}_i +  \sum_{i, j\neq i} \mathbf{z}_i^\intercal P \mathbf{z}_j \nonumber\\
&\text{\ \  subject to }~~ \mathbf{z}_i \in \{0,1\}^d, || \mathbf{z}_i ||_1 = k, ~\forall i
\label{eqn:k_energy_avg}
\end{align}
\normalsize
In the upper bound problem above, we consider the case where the pairwise cost matrix $P$ is a diagonal matrix of non-negative values \footnote{Note that we absorb the scaler factor $m$ from the definition of $P'$ and redefine $P = \diag(\lambda_1, \ldots, \lambda_d)$.}. \Cref{thm:equivalence} in the next subsection proves that finding the optimal solution of \Cref{eqn:k_energy_avg} is equivalent to finding the minimum cost flow solution of the flow network $G'$ illustrated in \Cref{fig:mcf} which can be solved efficiently and exactly in polynomial time. In practice, we use the efficient implementations from OR-Tools (Google Optimization Tools for combinatorial optimization problems) \cite{ortools} to solve the minimum cost flow problem per each mini-batch.  \vspace{-0.1em}
\subsection{Equivalence of problem \ref{eqn:k_energy_avg} to minimum cost flow}  
\begin{thm}\label{thm:equivalence}
The optimization problem in \Cref{eqn:k_energy_avg} can be solved by finding the minimum cost flow solution on the flow network G'.
\end{thm}
\begin{proof}
Suppose we construct a complete bipartite graph $G=(A \cup B, E)$ and create a directed graph $G'=(A \cup B \cup \{s, t\}, E')$ from $G$ by adding source $s$ and sink $t$ and directing all edges $E$ in $G$ from $A$ to $B$. We also add edges from $s$ to each vertex $a_p \in A$. For each vertex $b_q \in B$, we add $n_c$ number of edges to $t$. Edges incident to $s$ have capacity $u(s, a_p) = k$ and cost $v(s, a_p)=0$. The edges between $a_p \in A$ and $b_q \in B$ have capacity $u(a_p, b_q) =1$ and cost $v(a_p, b_q) = -\bfc_p[q]$. Each edge $r\in\{0,\ldots,n_c-1\}$ from $b_q \in B$ to $t$ has capacity $u((b_q, t)_r) = 1$ and cost $v((b_q, t)_r) = 2 \lambda_q r$. \Cref{fig:mcf} illustrates the flow network $G'$. The amount of flow to be sent from $s$ to $t$ is $n_ck$. 

Then we define the flow $\{f_z(e)\}_{e \in E'}$, indexed both by (1) a given configuration of $\bfz_1, \ldots, \bfz_{n_c}$ where each $\bfz_i \in \{0,1\}^d, ||\bfz_i||_1=k, \forall i $, and by (2) the edges of $G'$, below: 
\begin{align}
&(i)~ f_z(s, a_p) = k,~ (ii)~ f_z(a_p, b_q) = \bfz_p[q],\nonumber \\
&(iii)~ f_z((b_q, t)_r) = \begin{cases} 1&\textrm{for  } r< \sum_{p=1}^{n_c} \bfz_p[q]\\
0&\textrm{otherwise}\end{cases}
\label{eqn:flow_def}
\end{align}\vspace{-1em}

We first show the flow $f_z$ defined above is feasible for $G'$. The capacity constraints are satisfied by construction in \Cref{eqn:flow_def}, so we only need to check the flow conservation conditions. First, the amount of input flow at $s$ is $n_c k$ and the output flow from $s$ is $\sum_{a_p \in A} f_z(s, a_p) = \sum_{a_p \in A} k = n_c k$ which is equal. The amount of input flow to each vertex $a_p \in A$ is given as $k$ and the output flow is $\sum_{b_q \in B} f_z(a_p, b_q) = \sum_{q}^d \bfz_p[q] = ||\bfz_p||_1 = k$. \\

Let us denote the amount of input flow at a vertex $b_q \in B$ as $y_q = \sum_{p}^{n_c} \bfz_p[q]$. The output flow from the vertex $b_q$ is $\sum_{r=0}^{n_c - 1} f_z((b_q, t)_r) = \sum_{r=0}^{y_q-1} f_z((b_q, t)_r) + \sum_{r=y_q}^{n_c -1} f_z((b_q, t)_r) = y_q$ from \Cref{eqn:flow_def} ($iii$). The last condition to check is that the amount of input flow at $t$ is equal to the output flow at $s$. $\sum_{b_q \in B} \sum_{r=0}^{n_c - 1} f_z((b_q, t)_r) = \sum_{q=1}^d y_q = \sum_{q,p} \bfz_p[q] = n_c k$. This shows the construction of the flow $\{f_z(e)\}_{e \in E'}$ in \Cref{eqn:flow_def} is valid in $G'$.\\

Now denote $\{f_o(e)\}_{e \in E'}$ as the minimum cost flow solution of the flow network $G'$ which minimizes the total cost $\sum_{e \in E'} v(e) f_o(e)$. Denote the optimal flow from a vertex $a_p \in A$ to a vertex $b_q \in B$ as $\bfz_p'[q] := f_o(a_p, b_q)$. By the optimality of the flow $\{f_o(e)\}_{e \in E'}$, we have that $\sum_{e \in E'} v(e) f_o(e) \leq \sum_{e \in E'} v(e) f_z(e)$. By \Cref{lemma:fo}, the $lhs$ of the inequality is equal to $\sum_p -\bfc_p^\intercal \bfz'_p + \sum_{p_1\neq p_2} \bfz_{p_1}'^\intercal P \bfz'_{p_2}$. Also, by \Cref{lemma:fz}, the $rhs$ is equal to $\sum_p -\bfc_p^\intercal \bfz_p + \sum_{p_1\neq p_2} \bfz_{p_1}^\intercal P \bfz_{p_2}$.\\

Finally, we have that $\sum_p -\bfc_p^\intercal \bfz'_p + \sum_{p_1\neq p_2} \bfz_{p_1}'^\intercal P \bfz'_{p_2} \leq \sum_p -\bfc_p^\intercal \bfz_p + \sum_{p_1\neq p_2} \bfz_{p_1}^\intercal P \bfz_{p_2}, \forall \{\bfz_p\}$. Thus, we have proved that finding the minimum cost flow solution on the flow network $G'$ and translating the flows between each vertices between $A$ and $B$ as $\{\bfz_p'\}$, we can find the optimal solution to the optimization problem in \Cref{eqn:k_energy_avg}. 
\end{proof}

\begin{lemma}\label{lemma:fo}
For the minimum cost flow $\{f_o(e)\}_{e \in E'}$ of the network $G'$, we have that the total cost is $\sum_{e \in E'} v(e) f_o(e) = \sum_p -\bfc_p^\intercal \bfz'_p + \sum_{p_1\neq p_2} \bfz_{p_1}'^\intercal P \bfz'_{p_2}$. 
\end{lemma}
\begin{proof}
The total minimum cost $\sum_{e \in E'} v(e) f_o(e)$ is broken down as \vspace{-2em}

\small
\begin{align}
&\sum_{e \in E'} v(e) f_o(e) = \underbrace{\sum_{a_p \in A} v(s, a_p) f_o(s, a_p)}_{\text{flow from } s \text{ to } A} ~+ \nonumber\\
&\underbrace{\sum_{a_p \in A} \sum_{b_q \in B} v(a_p, b_q) f_o(a_p, b_q)}_{\text{flow from } A \text{ to } B} + \underbrace{\sum_{b_q \in B} \sum_{r=0}^{n_c-1} v((b_q, t)_r) f_o((b_q, t)_r)}_{\text{flow from } B \text{ to } t} \nonumber
\end{align}
\normalsize
Denote the amount of input flow at a vertex $b_q$ given the minimum cost flow $\{f_o(e)\}_{e \in E'}$ as $y_q' = \sum_p f_o(a_p, b_q) = \sum_{p}^{n_c} \bfz_p'[q]$. From the cost definition at the edges between $b_q$ and $t$, $v((b_q, t)_r) = 2 \lambda_q r$, and by the optimality of the minimum cost flow, we have that $f_o((b_q, t)_r) = 1 ~\forall r < y_q'$ and $f_o((b_q, t)_r) = 0 ~\forall r \geq y_q'$. Thus, the total cost is \vspace{-1em}
\small
\begin{align}
\sum_{e \in E'} v(e) f_o(e) &= 0 + \sum_{p}^{n_c} \sum_{q}^d -\bfc_p[q] \bfz_p'[q] + \sum_{b_q \in B}\sum_{r=0}^{y_q'-1} 2 \lambda_q r \nonumber\\
&= \sum_{p} -\bfc_p^\intercal \bfz_p' + \sum_{q} \lambda_q y_q'(y_q' - 1)\nonumber\\
&= \sum_{p} -\bfc_p^\intercal \bfz_p' + \sum_{q} \lambda_q {y_q'}^2-\sum_p \sum_q \lambda_q \bfz_{p}'[q]  \nonumber\\
&= \sum_{p} -\bfc_p^\intercal \bfz_p' + {\sum_{p} \bfz_{p}'}^\intercal P {\sum_{p} \bfz_{p}'} - \sum_ p \bfz_{p}'^\intercal P \bfz_{p}' \nonumber\\
&= \sum_{p} -\bfc_p^\intercal \bfz_p' + \sum_{p_1 \neq p_2} \bfz_{p_1}'^\intercal P \bfz_{p_2}' 
\end{align} 
\end{proof} 
\normalsize
\begin{lemma}\label{lemma:fz}
For the $\{f_z(e)\}_{e \in E'}$ defined as \Cref{eqn:flow_def} of the network $G'$, we have that the total cost is $\sum_{e \in E'} v(e) f_z(e) = \sum_p -\bfc_p^\intercal \bfz_p + \sum_{p_1\neq p_2} \bfz_{p_1}^\intercal P \bfz_{p_2}$. 
\end{lemma}
\begin{proof}
The proof is similar to \Cref{lemma:fo} except that we use the definition of the flow $\{f_z(e)\}_{e \in E'}$ in \Cref{eqn:flow_def} ($iii$) to reduce the cost of the flow from $B$ to $t$ to $\sum_{r=0}^{y_q' - 1} 2 \lambda_q r$.
\end{proof} \vspace{-1em}
\textbf{Time complexity} ~For $\lambda_q \ll n_c$, note that the worst case time complexity of finding the minimum cost flow (MCF) solution in the network $G'$ is $\mathcal{O}\left(\left(n_c+d\right)^2 n_c d \log{\left(n_c+d\right)}\right)$ \cite{goldberg1990finding}. In practice, however, it has been shown that implementation heuristics such as \emph{price updates}, \emph{price refinement}, \emph{push-look-ahead}, \cite{goldberg1997efficient} and \emph{set-relabel} \cite{bunnagel1998efficient} methods drastically improve the real-life performance. Also, we emphasize again that we solve the minimum cost flow problem only \emph{within the mini-batch} not on the entire dataset. We benchmarked the wall clock running time of the method at varying sizes of $n_c$ and $d$ and observed approximately linear time complexity in $n_c$ and $d$. \Cref{fig:mcf_runtime} shows the benchmark wall clock run time results.  \vspace{-0.7em}

\begin{figure}[ht]
\centering
\begin{tikzpicture}
\begin{axis}[
	legend pos={north west},
	xtick={64,128,256,512},
	xlabel={$d$},
	ylabel={Average wall clock run time (sec)},
	grid=major
]
\addplot coordinates {
	(64, 0.0235) (128, 0.0460) (256, 0.1017) (512, 0.1962)
};
\addplot coordinates {
	(64, 0.0473) (128, 0.0941) (256, 0.1903) (512, 0.3846)
};
\addplot coordinates {
	(64, 0.0975) (128, 0.1939) (256, 0.3917) (512, 0.7822)
};
\addplot coordinates {
	(64, 0.2026) (128, 0.4026) (256, 0.7958) (512, 1.6129)
};

\legend{$n_c=64$,$n_c=128$,$n_c=256$,$n_c=512$}
\end{axis}
\end{tikzpicture} \vspace{-2.5em}
\caption{Average wall clock run time of computing minimum cost flow on $G'$ per mini-batch using \cite{ortools}. In practice, the run time is approximately linear in $n_c$ and $d$. Each data point is averaged over 20 runs on machines with Intel Xeon E5-2650 CPU.}
\label{fig:mcf_runtime} \vspace{-2em}
\end{figure}
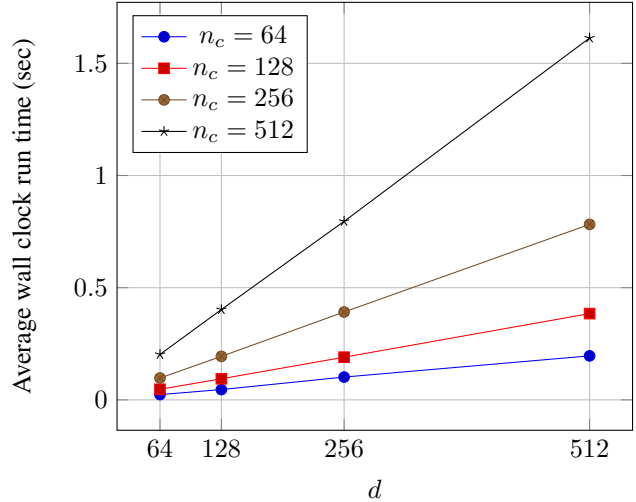

\makeatletter
\newcommand{\gettikzxy}[3]{
  \tikz@scan@one@point\pgfutil@firstofone#1\relax
  \edef#2{\the\pgf@x}
  \edef#3{\the\pgf@y}
}
\makeatother
\begin{figure}[ht]
\centering
\begin{tikzpicture}[every node/.style={scale=0.8},font=\tt]
	\SetUpEdge[lw=0.75pt,color=red,labelcolor=white]
	\GraphInit[vstyle=normal] 
	\SetGraphUnit{2}
	\SetVertexMath
	\tikzset{VertexStyle/.append  style={fill=red!50,minimum size=0.9cm}}
	\Vertex{s}
	\tikzset{VertexStyle/.append  style={fill=white}}
	\NOEA(s){a_1}
	\SO[unit=1.5](a_1){a_p}
	\path (a_1) -- (a_p) node [font=\huge, midway, sloped] {$\dots$};
	\SO[unit=1.9](a_p){a_{n_c}}	
	\path (a_p) -- (a_{n_c}) node [font=\huge, midway, sloped] {$\dots$};
	\SetUpEdge
	\tikzset{EdgeStyle/.style={-stealth}}
	\Edge(s)(a_1)	
	\Edge[label={k,0},style={sloped}](s)(a_p)
	\Edge(s)(a_{n_c})
	\NOEA[unit=2.8](a_p){b_1}
	\SO[unit=2.8](b_1){b_q}
	\path (b_1) -- (b_q) node [font=\huge, midway, sloped] {$\dots\dots$};
	\SO[unit=2.5](b_q){b_d}
	\path (b_q) -- (b_d) node [font=\huge, midway, sloped] {$\dots\dots$};
	\Edge(a_1)(b_1) \Edge(a_1)(b_q) \Edge(a_1)(b_d)
	\Edge(a_{n_c})(b_1) \Edge(a_{n_c})(b_q) \Edge(a_{n_c})(b_d)
	\Edge[label={1,$-c_p[1]$},style={sloped}](a_p)(b_1) \Edge[label={1,$-c_p[q]$},style={sloped}](a_p)(b_q) \Edge[label={1,$-c_p[d]$},style={sloped}](a_p)(b_d)
	\tikzset{VertexStyle/.append  style={fill=blue!50,minimum size=0.9cm}}
	\EA[unit=7.9](s){t}
	\Edge[style={bend left=70}](b_1)(t) \Edge[style={bend left=37}](b_1)(t) \Edge[style={bend left=30}](b_1)(t)
    \draw[opacity=0,line width=0cm] (b_1) to [bend left=70]  coordinate[pos=0.3] (mid_b_1_t_1)(t);
    \gettikzxy{(mid_b_1_t_1)}{\oneax}{\oneay}
    \draw[opacity=0,line width=0cm] (b_1) to [bend left=37]  coordinate[pos=0.3] (mid_b_1_t_2)(t);
    \gettikzxy{(mid_b_1_t_2)}{\onebx}{\oneby}
	\Edge[label={1,$2(n_c-1)\lambda_q$},style={bend left=40,sloped}](b_q)(t)
    \draw[opacity=0,line width=0cm] (b_q) to [bend left=40,sloped]  coordinate[pos=0.5] (mid_b_q_t_1)(t);
    \Edge[label={1,0},style={bend right=24}](b_q)(t)
    \gettikzxy{(mid_b_q_t_1)}{\qax}{\qay}
	\Edge[label={1,$2\lambda_q$},style={bend right=5,sloped}](b_q)(t)
    \draw[opacity=0,line width=0cm] (b_q) to [bend right=5,sloped]  coordinate[pos=0.5] (mid_b_q_t_2)(t);
    \gettikzxy{(mid_b_q_t_2)}{\qbx}{\qby}
	\Edge[style={bend right=5}](b_d)(t) \Edge[style={bend right=38}](b_d)(t) \Edge[style={bend right=45}](b_d)(t)
    \draw[opacity=0,line width=0cm] (b_d) to [bend right=5]  coordinate[pos=0.5] (mid_b_d_t_1)(t);
    \gettikzxy{(mid_b_d_t_1)}{\dax}{\day}
    \draw[opacity=0,line width=0cm] (b_d) to [bend right=38]  coordinate[pos=0.4] (mid_b_d_t_2)(t);
    \gettikzxy{(mid_b_d_t_2)}{\dbx}{\dby}
	\path ({(\qax+\qbx)/2}, \oneay) -- ({(\qax+\qbx)/2}, \oneby) node [font=\Large, midway, sloped] {$\dots$};
	\path ({(\qax+\qbx)/2}, \qay) -- ({(\qax+\qbx)/2}, \qby) node [font=\Large, midway, sloped] {$\dots$};
	\path ({(\qax+\qbx)/2}, \day) -- ({(\qax+\qbx)/2}, \dby) node [font=\Large, midway, sloped] {$\dots$};
\end{tikzpicture} \vspace{-2.3em}
\caption{Equivalent flow network diagram $G'$ for the optimization problem in \Cref{eqn:k_energy_avg}. Labeled edges show the capacity and the cost respectively. The amount of total flow to be sent is $n_c k$.} \label{fig:mcf}\vspace{-0.3em}
\end{figure}
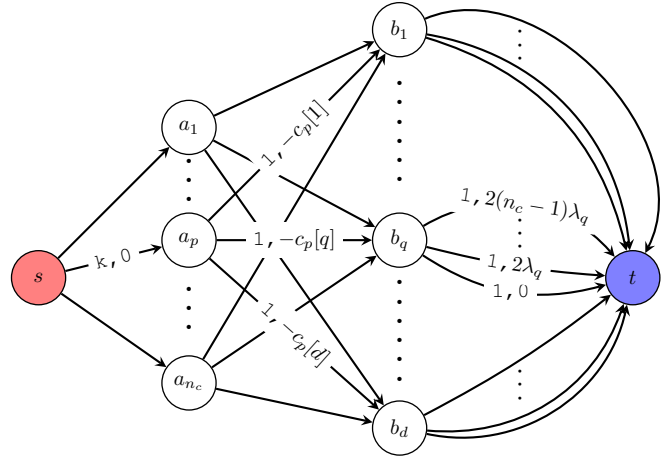

\subsection{Learning the embedding} \vspace{-0.3em}
\label{sec:method_embedding}
As the hash codes become more and more sparse, it becomes increasingly likely for hamming distances defined on binary codes \cite{hammingmetric, zhao2015deep} to become zero regardless of whether the input pair of data is similar or dissimilar. This phenomenon can be problematic when trained in a deep network because the back-propagation gradient would become zero and thus the embedding representations would not be updated at all. In this regard, we propose a distance function based on gated residual as shown in \Cref{eqn:h_distance}. This parameterization outputs zero distance only if the embedding representations of the two input data are identical at all the hash code activations. Concretely, given a pair of embedding vectors $f(\mathbf{x}_i; \bftheta), f(\mathbf{x}_j; \bftheta)$ and the corresponding binary k-sparse hash codes $\mathbf{h}_i, \mathbf{h}_j$, we define the following distance function $d_{ij}^{\text{ hash}}$ between the embedding vectors \vspace{-1.5em}

\small
\begin{align}
	d_{ij}^{\text{ hash}} = || \left(\mathbf{h}_i \lor \mathbf{h}_j\right) \odot \left(f(\mathbf{x}_i; \bftheta) - f(\mathbf{x}_j; \bftheta)\right) ||_1,
	\label{eqn:h_distance} \vspace{-1em}
\end{align}
\normalsize
where $\lor$ denotes the logical \emph{or} operation of the two binary hash codes and $\odot$ denotes the element-wise multiplication. Then, using the distance function above, we can define the following subproblems using any existing deep metric learning methods \cite{triplet, facenet, liftedstruct, npairs, facility}. Given a set of binary hash codes $\{\bfh_i\}_{i=1}^n$, we seek to solve the following subproblems where the task is to optimize the embedding representations so that similar pairs of data get hashed to the same buckets and dissimilar pairs of data get hashed to different buckets. In other words, we need similar pairs of data to have similar embedding representations indexed at the activated hash code dimensions and vice versa. In terms of the hash code optimization in \Cref{eqn:upperbound}, updating the network weight has the effect of tightening the bound gap $M(\bftheta)$.

\Cref{eqn:h_triplet} and \Cref{eqn:h_npairs} show the subproblems defined on the distance function above using \emph{Triplet} \cite{facenet} and \emph{Npairs} \cite{npairs} method respectively. We optimize these embedding subproblems by updating the network parameter $\bftheta$ via stochastic gradient descent using the subgradients $\frac{\partial \ell_{\text{metric}}(\bftheta; \bfh_{1,\ldots,n})}{\partial \bftheta}$ given the hash codes per each mini-batch. \vspace{-1.5em}

\small
\begin{align}
	&\minimize_{\bftheta}~ \underbrace{\frac{1}{|\mathcal{T}|} \sum_{(i,j,k) \in \mathcal{T}} [ d_{ij}^{\text{ hash}} + \alpha - d_{ik}^{\text{ hash}} ]_+}_{\ell_{\text{triplet}}(\bftheta; ~\bfh_{1,\ldots,n})} \nonumber\\
	&\text{\ \ subject to }~ || f(\mathbf{x; \bftheta}) ||_2 = 1,
	\label{eqn:h_triplet}
\end{align}\vspace{0.5em}
\normalsize 
where $\mathcal{T} = \{ (\bfx_i, \bfx_i^+, \bfx_i^-) \}_i$ is the set of triplets \cite{facenet}, and the embedding vectors are normalized onto unit hypersphere $||f(\mathbf{x}; \bftheta)||_2 = 1, ~\forall \mathbf{x} \in \mathcal{X}$. We also apply the \emph{semi-hard negative mining} procedure \cite{facenet} where hard negatives farther than the distance between the anchor and positives are mined within the mini-batch. In practice, since our method can be applied to any deep metric learning methods, we use existing deep metric learning implementations available in \href{https://www.tensorflow.org/versions/master/api_docs/python/tf/contrib/losses/metric_learning}{\texttt{tf.contrib.losses.metric\_learning}}. Similarly, we could also employ \emph{npairs} \cite{npairs} method, \vspace{-1.5em}

\small
\begin{align}
	\minimize_{\bftheta}& \underbrace{\frac{-1}{|\mathcal{P}|} \sum_{(i,j) \in \mathcal{P}} \log \frac{\exp(-d_{ij}^{\text{ hash}})}{\exp(-d_{ij}^{\text{ hash}}) + \sum_{k: y_k \neq y_i} \exp(-d_{ik}^{\text{ hash}}) }}_{\ell_{\text{npairs}}(\bftheta; ~\bfh_{1,\ldots,n})} \nonumber\\
	&+ \frac{\lambda}{m} \sum_{i} || f(\mathbf{x}_i; \bftheta) ||_2^2, 
	\label{eqn:h_npairs}
\end{align}
\normalsize 
where the \emph{npairs} mini-batch $B$ is constructed with positive pairs $(\mathbf{x}, \mathbf{x}^+)$ which are negative with respect to all other pairs. $B=\{(\mathbf{x}_1, \mathbf{x}_1^+), \ldots, (\mathbf{x}_n, \mathbf{x}_n^+))\}$ and $\mathcal{P}$ denotes the set of all positive pairs within the mini-batch. We use the existing implementation of npairs loss in Tensorflow as well. Note that even though the distance $d_{ij}^{\text{ hash}}$ is defined \emph{after} masking the embeddings with the union binary vector $(\mathbf{h}_i \lor \mathbf{h}_j)$, it's important to normalize or regularize the embedding representation \emph{before} the masking operations for the optimization stability due to the sparse nature of the hash codes. \vspace{0.5em}

\begin{algorithm}[!hbtp]
   \small
   \caption{Learning algorithm}
   \label{alg:procedure}
   \begin{algorithmic}[1]
      \INPUT $\bftheta_b^{\text{emb}}$ (pretrained metric learning base model); $\bftheta_d \in \mathbb{R}^d$
      \REQUIRE $\bftheta_f = [\bftheta_b, \bftheta_d]$
      \FOR{ $t=1,\ldots,$ MAXITER}
      \STATE Sample a minibatch $\{\bfx_j\}$
      \STATE Update the flow network $G'$ by recomputing the cost vectors for all classes in the minibatch\\
      $\bfc_i = \frac{1}{m} \sum_{k: y_k = i} f(\bfx_k; \bftheta_f)$
      \STATE Compute the hash codes $\{\bfh_i\}$ minimizing \Cref{eqn:k_energy_avg} via finding the minimum cost flow on $G'$
      \STATE Update the network parameter given the hash codes\\
      $\bftheta_f \leftarrow \bftheta_f - \eta^{(t)} \partial \ell_\text{metric}(\bftheta_f; ~\bfh_{1,\ldots,n_c}) / \partial \bftheta_f$
      \STATE Update stepsize $\eta^{(t)} \leftarrow$ ADAM rule \cite{adam}
   \ENDFOR
   \OUTPUT $\bftheta_{f}$ (final estimate);
\end{algorithmic}
\end{algorithm} \vspace{-0.2em}
\subsection{Query efficiency analysis} \vspace{-0.3em}
In this subsection, we examine the expectation and the variance of the query time speed up over linear search. Recall the properties of the compound hash code defined in \Cref{sec:formulation}, $\bfh \in \{0,1\}^d$ and $||\bfh||_1 = k$. Given $n$ such hash codes, we have that $\mathrm{Pr}(\bfh_i^\intercal \bfh_j = 0) = \binom{d-k}{k}/\binom{d}{k}$ assuming the hash code uniformly distributes the items throughout different buckets. For a given hash code $\bfh_q$, the number of retrieved data is $N_q = \sum_{i \neq q} \mathds{1}(\bfh_i^\intercal \bfh_q\neq 0)$. Then, the expected number of retrieved data is $\mathbb{E}[N_q] = (n-1)\left( 1- \binom{d-k}{k}/\binom{d}{k} \right)$. Thus, in contrast to linear search, the expected speedup factor (SUF) under perfectly uniform distribution of the hash code is \vspace{-2em}

\small
\begin{align}
\mathbb{E}[\text{SUF}] = \left( 1- \frac{\binom{d-k}{k}}{\binom{d}{k}} \right)^{-1} 
\label{eqn:suf}
\end{align} 
\normalsize
In the case where $d \gg k$, the speedup factor approaches $\left(\frac{d}{k^2}\right)$. Similarly, we have that the variance is $V[N_q] = (n-1) \Big{(} 1- \binom{d-k}{k}/\binom{d}{k} \Big{)}\binom{d-k}{k}/\binom{d}{k}$. \vspace{-0.1em}

\section{Implementation details} \vspace{-0.3em}
\label{sec:implementation}
\textbf{Network architecture} ~In our experiments, we used the NIN \cite{NIN}  architecture (denote the parameters as $\bftheta_b$) with \emph{leaky relu} \cite{xu2015empirical} with $\alpha=5.5$ as activation function and trained Triplet embedding network with semi-hard negative mining \cite{facenet} and Npairs network \cite{npairs} from scratch as the base model. We snapshot the network weights ($\bftheta_b^{\text{emb}}$) of the learned base model. Then we replace the last layer in ($\bftheta_b^{\text{emb}}$) with a randomly initialized $d$ dimensional fully connected projection layer ($\bftheta_d$) and finetune the hash network (denote the parameters as $\bftheta_f = [\bftheta_b, \bftheta_d]$). \Cref{alg:procedure} summarizes the training procedure in detail. \vspace{-0.2em}

\textbf{Hash table construction and query} ~We use the learned hash network $\bftheta_f$ and apply \Cref{eqn:hash_function} to convert a hash data $\bfx_i$ into the hash code $\bfh(\bfx_i; \bftheta_f)$ and use the base embedding network $\bftheta_b^{\text{emb}}$ to convert the data into the embedding representation $f(\bfx_i; \bftheta_b^{\text{emb}})$. Then, the embedding representation is hashed to buckets corresponding to the $k$ set bits in the hash code. We use the similar procedure and convert a query data $\bfx_q$ into the hash code $\bfh(\bfx_q; \bftheta_f)$ and into the embedding representation $f(\bfx_q; \bftheta_b^{\text{emb}})$. Once we retrieve the union of all bucket items indexed at the $k$ set bits in the hash code, we apply a reranking procedure \cite{survey_learningtohash} based on the euclidean distance in the embedding representation space. \vspace{-0.2em}

\textbf{Evaluation metrics} ~We report our accuracy results using precision@k (Pr@k) and normalized mutual information (NMI) \cite{manningbook} metrics. Precision@k is computed based on the reranked ordering (described above) of the retrieved items from the hash table. We evaluate NMI, when the code sparsity is set to $k=1$, treating each bucket as individual clusters. In this setting, NMI becomes perfect, if each bucket has perfect class purity (pathologically putting one item per each bucket is prevented by construction since $d \ll n$). We report the speedup results by comparing the number of retrieved items versus the total number of data (exhaustive linear search) and denote this metric as SUF. As the hash code becomes uniformly distributed, SUF metric approaches the theoretical expected speedup in \Cref{eqn:suf}. \Cref{fig:SUF} shows that the measured SUF of our method closely follows the theoretical upper bound in contrast to other methods.\vspace{-0.5em}

\section{Experiments} \vspace{-0.3em}
\label{sec:exp}
We report our results on Cifar-100 \cite{cifar100} and ImageNet \cite{imagenet} datasets and compare the accuracy against several baseline methods. First baseline methods are the state of the art deep metric learning models \cite{facenet, npairs} performing an exhaustive linear search over the whole dataset given a query data. Another baselines are the Binarization transform \cite{agrawal2014,zhai2017} methods where the dimensions of the hash code corresponding to the top $k$ dimensions of the embedding representation are set. We also perform vector quantization \cite{survey_learningtohash} on the learned embedding representation from the deep metric learning methods above on the entire dataset and compute the hash code based on the indices of the $k$ nearest centroids. \vspace{-0.4em} `Triplet' and `Npairs' denotes the deep metric learning base models performing an exhaust linear search per each query. `Th' denotes the binarization transform baseline, `VQ' denotes the vector quantization baseline.

\subsection{Cifar-100} \vspace{-0.3em}
Cifar-100 \cite{cifar100} dataset has $100$ classes. Each class has $500$ images for \emph{train} and $100$ images for \emph{test}.  Given a query image from \emph{test}, we experiment the search performance both when the hash table is constructed from \emph{train} and from \emph{test}.  We subtract the per-pixel mean of training images across all the images and augmented the dataset by zero-padding 4 pixels on each side, randomly cropping $32\times 32$, and applying random horizontal flipping.  The batch size is set to $128$. The metric learning base model is trained for $175$k iterations, and learning rate decays to $0.1$ of previous learning rate after $100$k iterations.  We finetune the base model for $70$k iterations and decayed the learning rate to $0.1$ of previous learning rate after $40$k iterations. \Cref{tab:cifar_triplet_train_test} show results using the triplet network with $d\!=\!256$ and \Cref{tab:cifar_npairs_train_test} show results using the npairs network with $d\!=\!64$.  The results show that our method not only outperforms search accuracies of the state of the art deep metric learning base models but also provides up to $98\times$ speed up over exhaustive search.

\begin{table}[htbp]
\centering
\tiny
\begin{adjustbox}{max width=\columnwidth}
    \begin{tabular}{cc cccc cccc}
    \cmidrule[1pt](r){1-6} \cmidrule[1pt](r){7-10}
        &      &  \multicolumn{4}{c}{\emph{train}}                                & \multicolumn{4}{c}{\emph{test}}\\
    \cmidrule[1pt](r){1-6} \cmidrule[1pt](r){7-10}
                &Method& SUF& Pr@1  & Pr@4  & Pr@16 & SUF   & Pr@1  & Pr@4  & Pr@16 \\
    \cmidrule[1pt](r){1-6} \cmidrule[1pt](r){7-10}
                & Triplet & 1.00 &  62.64 & 61.91 & 61.22 & 1.00 &  56.78 & 55.99 & 53.95 \\
    \cmidrule[1pt](r){1-6} \cmidrule[1pt](r){7-10}
    \multirow{3}{*}{$k\!=\!1$}
    		  & Triplet-Th  & 43.19  &  61.56 & 60.24 & 58.23 & 41.21 &  54.82 & 52.88 & 48.03 \\ 
              & Triplet-VQ  & 40.35 & 62.54 & 61.78 & 60.98  & 22.78 & 56.74 & 55.94 & 53.77 \\ 
              & Triplet-Ours & \textbf{97.77} &  \textbf{63.85} & \textbf{63.40}  & \textbf{63.39} & \textbf{97.67} &  \textbf{57.63} & \textbf{57.16} & \textbf{55.76} \\
    \cmidrule[1pt](r){1-6} \cmidrule[1pt](r){7-10}
    \multirow{3}{*}{$k\!=\!2$}
    		  & Triplet-Th  & 15.34 &  62.41 & 61.68 & 60.89 & 14.82 &  56.55 & 55.62 & 52.90  \\  
              & Triplet-VQ  & 6.94 &  62.66 & 61.92 & 61.26 & 5.63 &  56.78 & 56.00    & 53.99 \\ 
              & Triplet-Ours & \textbf{78.28} &  \textbf{63.60}  & \textbf{63.19} & \textbf{63.09}  & \textbf{76.12} &  \textbf{57.30}  & \textbf{56.70}  & \textbf{55.19} \\
    \cmidrule[1pt](r){1-6} \cmidrule[1pt](r){7-10}
    \multirow{3}{*}{$k\!=\!3$}
    		  & Triplet-Th  & 8.04 &  62.66 & 61.88 & 61.16 & 7.84 &  56.78 & 55.91 & 53.64 \\ 
              & Triplet-VQ  & 2.96 &  62.62 & 61.92 & 61.22 & 2.83 &  56.78 & 55.99 & 53.95 \\  
              & Triplet-Ours & \textbf{44.36} &  \textbf{62.87} & \textbf{62.22} & \textbf{61.84} & \textbf{42.12} & \textbf{56.97} & \textbf{56.25} & \textbf{54.40}  \\
    \cmidrule[1pt](r){1-6} \cmidrule[1pt](r){7-10}
    \multirow{3}{*}{$k\!=\!4$}
     		  & Triplet-Th  & 5.00 &  62.66 & 61.94 & 61.24  & 4.90 & 56.84 & 56.01 & 53.86 \\  
              & Triplet-VQ  & 1.97 & 62.62 & 61.91 & 61.22  & 1.91 & 56.77 & 55.99 & 53.94 \\ 
              & Triplet-Ours & \textbf{16.52} &  \textbf{62.81} & \textbf{62.14} & \textbf{61.58} & \textbf{16.19} &  \textbf{57.11} & \textbf{56.21} & \textbf{54.20}  \\
    \cmidrule[1pt](r){1-6} \cmidrule[1pt](r){7-10}
    \end{tabular}
\end{adjustbox} \vspace{-1em}
    \caption{Results with Triplet network with hard negative mining. Querying Cifar-100 \emph{test} data against hash tables built on \emph{train} set and on \emph{test} set.}
\label{tab:cifar_triplet_train_test} \vspace{-1em}
\end{table}

\begin{table}[htbp]
\centering
\tiny
\begin{adjustbox}{max width=\columnwidth}
    \begin{tabular}{cc cccc cccc}
    \cmidrule[1pt](r){1-6} \cmidrule[1pt](r){7-10}
        &      &  \multicolumn{4}{c}{\emph{train}}                                & \multicolumn{4}{c}{\emph{test}}\\
    \cmidrule[1pt](r){1-6} \cmidrule[1pt](r){7-10}
                &Method& SUF& Pr@1  & Pr@4  & Pr@16 & SUF   & Pr@1  & Pr@4  & Pr@16 \\
    \cmidrule[1pt](r){1-6} \cmidrule[1pt](r){7-10}
                 & Npairs & 1.00 &61.78  & 60.63 & 59.73 &1.00  & 57.05 & 55.70 & 53.91 \\
    \cmidrule[1pt](r){1-6} \cmidrule[1pt](r){7-10}
    \multirow{3}{*}{$k\!=\!1$}
        & Npairs-Th  & 13.65 & 60.80  & 59.49 & 57.27  & 12.72 &  54.95 & 52.60 & 47.16 \\  
         & Npairs-VQ  & 31.35 &  61.22 & 60.24 & 59.34  & 34.86 & 56.76 & 55.35 & 53.75 \\ 
         & Npairs-Ours & \textbf{54.90} &  \textbf{63.11} & \textbf{62.29} & \textbf{61.94} & \textbf{54.85} & \textbf{58.19} & \textbf{57.22} & \textbf{55.87} \\
    \cmidrule[1pt](r){1-6} \cmidrule[1pt](r){7-10}
    \multirow{3}{*}{$k\!=\!2$}
        & Npairs-Th  & 5.36  &  61.65 & 60.50  & 59.50   & 5.09 &  56.52 & 55.28 & 53.04 \\ 
         & Npairs-VQ  & 5.44  &  61.82 & 60.56 & 59.70 & 6.08 &  57.13 & 55.74 & 53.90  \\  
         & Npairs-Ours & \textbf{16.51} &  \textbf{61.98} & \textbf{60.93} & \textbf{60.15} & \textbf{16.20} &  \textbf{57.27} & \textbf{55.98} & \textbf{54.42} \\ 
    \cmidrule[1pt](r){1-6} \cmidrule[1pt](r){7-10}
    \multirow{3}{*}{$k\!=\!3$}
        & Npairs-Th  & 3.21 &  61.75 & 60.66 & 59.73  & 3.10 &  56.97 & 55.56 & 53.76 \\  
         & Npairs-VQ  & 2.36  &  61.78 & 60.62 & 59.73 & 2.66 &  57.01 & 55.69 & 53.90  \\  
         & Npairs-Ours & \textbf{7.32} &  \textbf{61.90}  & \textbf{60.80}  & \textbf{59.96} & \textbf{7.25} &  \textbf{57.15} & \textbf{55.81} & \textbf{54.10}  \\
    \cmidrule[1pt](r){1-6} \cmidrule[1pt](r){7-10}
    \multirow{3}{*}{$k\!=\!4$}
        & Npairs-Th  & 2.30 & 61.78 & 60.66 & 59.75  & 2.25 &  57.02 & 55.64 & 53.88 \\ 
         & Npairs-VQ  & 1.55 &  61.78 & 60.62 & 59.73  & 1.66 & 57.03 & 55.70  & 53.91 \\  
         & Npairs-Ours & \textbf{4.52} & \textbf{61.81} & \textbf{60.69} & \textbf{59.77} & \textbf{4.51} & \textbf{57.15} & \textbf{55.77} & \textbf{54.01} \\
    \cmidrule[1pt](r){1-6} \cmidrule[1pt](r){7-10}
    \end{tabular}
\end{adjustbox} \vspace{-1em}
    \caption{Results with Npairs \cite{npairs} network. Querying Cifar-100 \emph{test} data against hash tables built on \emph{train} set and on \emph{test} set.}
\label{tab:cifar_npairs_train_test} \vspace{-1em}
\end{table}

\subsection{ImageNet}  \vspace{-0.3em}
ImageNet ILSVRC-2012 \cite{imagenet} dataset has $1,000$ classes and comes with \emph{train} ($1,281,167$ images) and \emph{val} set ($50,000$ images). We use the first nine splits of \emph{train} set to train our model, the last split of \emph{train} set for validation, and use \emph{validation} dataset to test the query performance. We use the images downsampled to $32\times32$ from \cite{imgnet-down}. Preprocessing step is identical with cifar-100 and we used the pixel mean provided in the dataset. The batch size for the metric learning base model is set to $512$ and is trained for $450$k iterations, and learning rate decays to $0.3$ of previous learning rate after each $200$k iterations. When we finetune npairs base model for $d\!=\!512$, we set the batch size to $1024$ and total iterations to $35$k with decaying the learning rate to $0.3$ of previous learning rate after each $15$k iterations. When we finetune the triplet base model for $d\!=\!256$, we set the batch size to $512$ and total iterations to $70$k with decaying the learning rate to $0.3$ of previous learning rate after each $30$k iterations. Our results in \Cref{tab:imagenet_npairs} and \Cref{tab:imagenet_triplet} show that our method outperforms the state of the art deep metric learning base models in search accuracy while providing up to $478\times$ speed up over exhaustive linear search. \Cref{tab:NMI} compares the NMI metric and shows that the hash table constructed from our representation yields buckets with significantly better class purity on both datasets and on both methods. 

\begin{table}[h]
\centering
\footnotesize
\begin{adjustbox}{max width=\columnwidth}
    \begin{tabular}{ccccccc}
        \toprule[1pt]
        & Method & SUF &  Pr@1  & Pr@4  & Pr@16 \\\midrule[1pt]
	 &Npairs&     1.00&                 15.73&                 13.75&                  11.08\\
        \midrule[1pt]
    \multirow{3}{*}{$k\!=\!1$} 
        &Npairs-Th&     1.74&                    15.06&                 12.92&                   9.92\\
        &Npairs-VQ&    451.42&                   15.20&      	       13.27&                  10.96\\
        &Npairs-Ours&   \textbf{478.46}&      \textbf{16.95}&      \textbf{15.27}&     \textbf{13.06}\\ \midrule[1pt]
    \multirow{3}{*}{$k\!=\!2$}
        &Npairs-Th&     1.18&                    15.70&                 13.69&                  10.96\\
        &Npairs-VQ&   116.26&                  15.62&                 13.68&                  11.15\\
        &Npairs-Ours&   \textbf{116.61}&                \textbf{16.40}&             \textbf{14.49}&              \textbf{12.00}\\ \midrule[1pt]
    \multirow{3}{*}{$k\!=\!3$}
        &Npairs-Th&     1.07&                   15.73&                 13.74&                  11.07\\
        &Npairs-VQ&    \textbf{55.80}&                  15.74&                 13.74&                  11.12\\
        &Npairs-Ours&    53.98&                   \textbf{16.24}&                 \textbf{14.32}&                  \textbf{11.73}\\ \bottomrule[1pt]
    \end{tabular}
\end{adjustbox} \vspace{-0.2em}
\caption{Results with Npairs \cite{npairs} network. Querying ImageNet \emph{val} data against hash table built on \emph{val} set.}
\label{tab:imagenet_npairs}
\end{table}

\begin{table}[h]
\centering
\footnotesize
\begin{adjustbox}{max width=\columnwidth}
    \begin{tabular}{cccccc}
    \toprule[1pt]
        & Method & SUF &  Pr@1  & Pr@4  & Pr@16 \\\midrule[1pt]
        & Triplet& 1.00&   10.90 & 9.39  &  7.45 \\ \midrule[1pt]
    \multirow{3}{*}{$k\!=\!1$}
     & Triplet-Th &  18.81&  10.20&  8.58&  6.50\\
       & Triplet-VQ & 146.26&  10.37& 8.84& 6.90\\
        & Triplet-Ours&   \textbf{221.49}&    \textbf{11.00}& \textbf{9.59}& \textbf{7.83} \\ \midrule[1pt]
    \multirow{3}{*}{$k\!=\!2$}
        &Triplet-Th&     6.33&                  10.82&                  9.30&                   7.32\\
    &Triplet-VQ&    32.83&                 10.88&                  9.33&                   7.39\\
        &Triplet-Ours&    \textbf{60.25}&                   \textbf{11.10}&                 \textbf{9.64}&                   \textbf{7.73}\\ \midrule[1pt]
    \multirow{3}{*}{$k\!=\!3$}
        &Triplet-Th&     3.64&                 10.87&                  9.38&                   7.42\\
        &Triplet-VQ&    13.85&                    10.90&                  9.38&                   7.44\\
        &Triplet-Ours&    \textbf{27.16}&               \textbf{11.20}&               \textbf{9.55}&                \textbf{7.60}\\ \bottomrule[1pt]
    \end{tabular}
\end{adjustbox} \vspace{-0.2em}
\caption{Results with Triplet network with hard negative mining. Querying ImageNet \emph{val} data against hash table built on \emph{val} set.}
\label{tab:imagenet_triplet} 
\end{table}

\begin{table}[h]
    \vspace{1em}
	\centering
	\footnotesize
	\begin{adjustbox}{max width=\columnwidth}
		\begin{tabular}{cccc}
            \toprule[1pt]
                              & \multicolumn{2}{c}{Cifar-100} & ImageNet \\ \cmidrule[1pt](r){2-3} \cmidrule[1pt](r){4-4}
                              & train         & test          & val      \\ \midrule[1pt]
			Triplet-Th   & 68.20         & 54.95         & 31.62    \\
			Triplet-VQ   & 76.85         & 62.68         & 45.47    \\
            Triplet-Ours & \textbf{89.11}         & \textbf{68.95}         & \textbf{48.52}    \\ \midrule[1pt]
			Npairs-Th   & 51.46         & 44.32         & 15.20    \\
			Npairs-VQ   & 80.25         & 66.69         & 53.74    \\
            Npairs-Ours & \textbf{84.90}         & \textbf{68.56}         & \textbf{55.09}    \\ \bottomrule[1pt]
		\end{tabular}
	\end{adjustbox} \vspace{-0.2em}
	\caption{Hash table NMI for Cifar-100 and Imagenet.}
	\label{tab:NMI} 
\end{table}

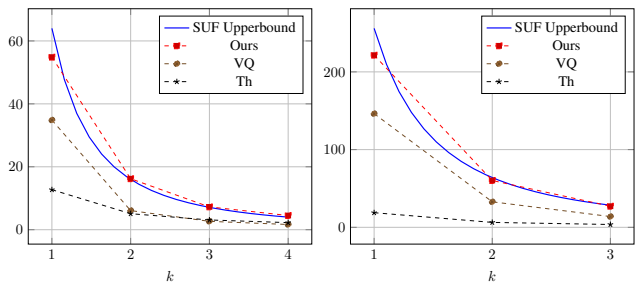
\begin{figure}[h]
    \vspace{-2em}
	\centering
	\hspace{-1.2em}
	\subfigure{
		\begin{tikzpicture}[scale=.55]
		  \begin{axis}[ 
				legend pos={north east},
				xtick={1,2,3,4},
			    	xlabel=$k$,
			   	domain=1:4,
				grid=major] 
		\addplot+[samples=20, thick, mark=none] {64*x^-2};
		\addlegendentry{SUF Upperbound};
		\addplot+[dashed]
		    coordinates {
		    (1,54.85)(2,16.20)(3,7.25)(4,4.51)
		    };
		\addlegendentry{Ours};
		\addplot+[dashed]
		    coordinates {
		    (1,34.86)(2,6.08)(3,2.66)(4,1.66)
		    };
		\addlegendentry{VQ};
		\addplot+[dashed]
		    coordinates {
		    (1,12.72)(2,5.09)(3,3.10)(4,2.25)
		    };
		\addlegendentry{Th};
		  \end{axis}
		\end{tikzpicture}
	\begin{tikzpicture}[scale=.55]
	  \begin{axis}[ 
			legend pos={north east},
			xtick={1,2,3},
		   	 xlabel=$k$,
		   	domain=1:3,
			grid=major] 
	\addplot+[samples=20, thick, mark=none] {256*x^-2};
	\addlegendentry{SUF Upperbound};
    \addplot+[dashed] coordinates {
	    (1,221.49)(2,60.25)(3,27.16)
	    };
	\addlegendentry{Ours};
    \addplot+[dashed] coordinates {
	    (1,146.26)(2,32.83)(3,13.85)
	    };
	\addlegendentry{VQ};
    \addplot+[dashed] coordinates {
	    (1,18.81)(2,6.33)(3,3.64)
	    };
	\addlegendentry{Th};
	  \end{axis}
	\end{tikzpicture}
	} \vspace{-2em}
	\caption{SUF metric for on Cifar-100 and ImageNet respectively.}
	\label{fig:SUF}  \vspace{-0.5em}
\end{figure}

\section{Conclusion} \vspace{-0.5em}
We have presented a novel end-to-end optimization algorithm for jointly learning a quantizable embedding representation and the sparse binary hash code which then can be used to construct a hash table for efficient inference. We also show an interesting connection between finding the optimal sparse binary hash code and solving a \emph{minimum cost flow} problem. Our experiments show that the proposed algorithm not only achieves the state of the art search accuracy outperforming the previous state of the art deep metric learning approaches \cite{facenet, npairs} but also provides up to $98\times$ and $478\times$ search speedup on Cifar-100 and ImageNet datasets respectively.

\section*{Acknowledgements} \vspace{-0.5em}
We would like to thank Zhen Li at Google Research for helpful discussions and anonymous reviewers for their constructive comments. This work was partially supported by Kakao, Kakao Brain and Basic Science Research Program through the National Research Foundation of Korea (NRF) (2017R1E1A1A01077431). Hyun Oh Song is the corresponding author.

\newpage
\bibliography{icml} 
\bibliographystyle{icml2018}
\end{document}